\documentclass[letterpaper, 10 pt, conference]{ieeeconf}  
\IEEEoverridecommandlockouts                              
\usepackage{times}
\usepackage[pdftex]{graphicx}
\usepackage{wrapfig}
\usepackage{amsmath,amsfonts,amssymb}
\usepackage{prettyref}
\usepackage{mathrsfs}
\usepackage{graphicx}
\usepackage{wrapfig}
\usepackage{xcolor}
\usepackage{subfig}
\usepackage{MnSymbol}
\usepackage{multirow}
\usepackage{booktabs}
\usepackage{algorithm}
\usepackage{algpseudocode}

\newtheorem{theorem}{Theorem}[section]
\newtheorem{lemma}[theorem]{\TE{Lemma}}

\algnewcommand{\LineComment}[1]{\State \(\triangleright\) #1}
\algdef{SE}[DOWHILE]{Do}{doWhile}{\algorithmicdo}[1]{\algorithmicwhile\ #1}
\newcommand*{\colorboxed}{}
\def\colorboxed#1#{%
  \colorboxedAux{#1}%
}
\newcommand*{\colorboxedAux}[3]{%
  \begingroup
    \colorlet{cb@saved}{.}%
    \color#1{#2}%
    \boxed{%
      \color{cb@saved}%
      #3%
    }%
  \endgroup
}

\newrefformat{Fig}{Figure~\ref{#1}}
\newrefformat{fig}{Figure~\ref{#1}}
\newrefformat{par}{Section~\ref{#1}}
\newrefformat{appen}{Appendix~\ref{#1}}
\newrefformat{sec}{Section~\ref{#1}}
\newrefformat{sub}{Section~\ref{#1}}
\newrefformat{table}{Table~\ref{#1}}
\newrefformat{alg}{Algorithm~\ref{#1}}
\newrefformat{Alg}{Algorithm~\ref{#1}}
\newrefformat{Def}{Definition~\ref{#1}}
\newrefformat{Thm}{Theorem~\ref{#1}}
\newrefformat{Lem}{Lemma~\ref{#1}}
\newrefformat{step}{Step~\ref{#1}}
\newrefformat{ln}{Line~\ref{#1}}
\newrefformat{eq}{Equation~\ref{#1}}
\newrefformat{pb}{Problem~\ref{#1}}
\newrefformat{it}{Item~\ref{#1}}
\newrefformat{te}{Term~\ref{#1}}
\def\Eqref Eq:#1:{\eqref{eq:#1}}
\newrefformat{Eq}{Equation~\Eqref#1:}

\newcommand{\E}[1]{\mathbf{#1}}
\newcommand{\TE}[1]{\textbf{#1}}

\newcommand{\TWOC}[2]{\left(\setlength{\arraycolsep}{1pt}\begin{array}{c}#1 \\ #2\end{array}\right)}

\newcommand{\THREE}[3]{\left(\setlength{\arraycolsep}{1pt}\begin{array}{ccc}{#1} & {#2} & {#3}\end{array}\right)}

\newcommand{\THREEC}[3]{\left(\setlength{\arraycolsep}{1pt}\begin{array}{c}#1 \\ #2 \\ #3\end{array}\right)}

\newcommand{\FOURC}[4]{\left(\setlength{\arraycolsep}{1pt}\begin{array}{c}#1 \\ #2 \\ #3 \\ #4\end{array}\right)}

\newcommand{\SIX}[6]{\left(\setlength{\arraycolsep}{1pt}\begin{array}{cccccc}{#1} & {#2} & {#3} & {#4} & {#5} & {#6}\end{array}\right)}

\newcommand{\MTT}[4]{\left(\setlength{\arraycolsep}{1pt}\begin{array}{cc}#1 & #2 \\ #3 & #4\end{array}\right)}

\newcommand{\tr}[1]{\E{tr}(#1)}

\newcommand{\fmin}[1]{\underset{#1}{\E{min}}}
\newcommand{\fmax}[1]{\underset{#1}{\E{max}}}
\newcommand{\argmin}[1]{\underset{#1}{\E{argmin}}}
\newcommand{\argminP}[1]{\E{argmin}}
\newcommand{\argmax}[1]{\underset{#1}{\E{argmax}}}
\newcommand{\argmaxP}[1]{\E{argmax}}
\newcommand{\ST}{\E{s.t.}}

\definecolor{darkgreen}{HTML}{186a3b}
\newcommand{\dd}{\E{d}}
\newcommand{\uu}{\E{u}}
\newcommand{\vv}{\E{v}}
\newcommand{\ff}{\E{f}}
\newcommand{\nn}{\E{n}}
\newcommand{\xx}{\E{x}}
\newcommand{\ww}{\E{w}}
\newcommand{\WW}{\E{W}}
\newcommand{\pxx}{\E{s}}
\newcommand{\KK}{\mathbb{K}}
\newcommand{\LL}{\mathcal{L}}
\newcommand{\CC}{\mathcal{C}}
\newcommand{\CU}{\E{C}}
\newcommand{\MM}{\mathcal{M}}
\newcommand{\MMM}{\mathbb{M}}
\newcommand{\DD}{\mathcal{D}}
\newcommand{\TT}{\mathcal{T}}
\newcommand{\NN}{\mathcal{N}}
\newcommand{\KDNode}{\mathbb{N}}
\newcommand{\SSS}{\mathcal{S}}
\newcommand{\Metric}{\mathbb{W}}
\newcommand{\hull}[1]{\E{ConvexHull}(#1)}
\newcommand{\TTMAT}{\THREE{x\xx\times}{y\xx\times}{z\xx\times}}
\newcommand{\MMMAT}{\xx\xx^T}
\newcommand{\GG}{\E{g}}
\newcommand{\GGG}{\E{G}}
\newcommand{\TR}{\E{tr}}

\newcommand{\stress}{\boldsymbol{\sigma}}
\newcommand{\deform}{\boldsymbol{\epsilon}}
\newcommand{\BEMA}{\mathcal{A}}
\newcommand{\BEMB}{\mathcal{B}}
\newcommand{\id}{\E{I}}
\newcommand{\lmt}{\lim_{\epsilon\to 0}}

\makeatletter
\IEEEtriggercmd{\reset@font\normalfont\fontsize{7.0pt}{7.2pt}\selectfont}
\makeatother
\IEEEtriggeratref{1}

\title{\large\bf Generating Optimal Grasps Under A Stress-Minimizing Metric \vspace{-10px}}
\author{Zherong Pan$^{1}$, Xifeng Gao$^{2}$, and Dinesh Manocha$^{3}$  \\
\vspace{-60px}
\thanks{$^{1}$Zherong is with Department of Computer Science, University of North Carolina at Chapel Hill. \{zherong@cs.unc.edu\}}
\thanks{$^{2}$Xifeng Gao is with Department of Computer Science, Florida State University.
\{gao@cs.fsu.edu\}}
\thanks{$^{3}$Dinesh Manocha is with Department of Computer Science and Electrical \& Computer Engineering, University of Maryland at College Park. \{dm@cs.umd.edu\}}}

\begin{document}
\maketitle
\thispagestyle{empty}
\pagestyle{empty}

\begin{abstract}
We present stress-minimizing (SM) metric, a new metric of grasp qualities. Unlike previous metrics that ignore the material of target objects, we assume that target objects are made of homogeneous isotopic materials. SM metric measures the maximal resistible external wrenches without causing fracture in the target objects. Therefore, SM metric is useful for robot grasping valuable and fragile objects. In this paper, we analyze the properties of this new metric, propose grasp planning algorithms to generate globally optimal grasps maximizing the SM metric, and compare the performance of the SM metric and a conventional metric. Our experiments show that SM metric is aware of the geometries of target objects while the conventional metric are not. We also show that the computational cost of the SM metric is on par with that of the conventional metric.
\end{abstract}
\section{\label{intro}Introduction}
The performance and properties of (asymptotically) optimal grasp planning algorithms heavily depend on the type of grasp quality metrics. A summary of these metrics can be found in \cite{roa2015grasp}. Usual requirements for a good grasp include: force closure, small contact force magnitudes, the preference of normal forces over frictional forces, higher resilience to external wrenches. The type of a metric not only reflects the requirements of an application, but also incurs different planning algorithms. For example, the $Q_{1,\infty}$ metrics are submodular, which allows fast discrete grasp point selection \cite{schulman2017grasping}. The $Q_1$ metric has an optimizable lower-bound, which allows an optimization-based grasp planning algorithm \cite{Dai2018} to jointly search for grasp points and grasp poses.

However, all the metrics considered so far take a common assumption about the target object: these objects are of infinite stiffness and will never be broken. As a result, we can assume that the target object is a rigid body and all the forces and torques are applied on the center-of-mass, which greatly simplify the computation and analysis of metrics. However, this assumption does not hold when grasping fragile objects where certain weak parts of an object should not be touched to avoid possible fractures. The grasp of fragile objects have been considered in prior works \cite{6697039,8344798}, which attempt to avoid breaking objects by developing safer grippers. Instead, we argue that, in addition to better robot hardware, a new metric is needed to guide the grasp planning algorithm, so that we can find a grasp that is most unlikely to break the object.

\TE{Main Results:} We present a novel metric that relaxes the infinite stiffness assumption and takes the material of the target object into consideration. Based on the theory of brittle fracture \cite{francfort1998revisiting}, we formulate the set of external wrenches that can be resisted without causing fractures in the object. Similar to the $Q_1$ metric \cite{219918,schulman2017grasping}, our metric then measures the size of the space of resistible wrenches. We call this new metric stress-minimizing (SM) metric $Q_{SM}$. We show that, using boundary element method (BEM) \cite{cruse1977numerical}, $Q_{SM}$ can be computed efficiently, given a closed surface triangle mesh of the object and a set of contact points. The cost of computing $Q_{SM}$ is on par with that of computing $Q_1$. In addition, we show that conventional optimization-based grasp planning algorithms \cite{schulman2017grasping} can be modified to search for globally optimal grasps maximizing $Q_{SM}$.

The rest of the paper is organized as follows. We review related work in \prettyref{sec:results} and formulation the problem of grasp planning in \prettyref{sec:problem}. The formulation of $Q_{SM}$ and its properties are summarized in \prettyref{sec:metric}. Grasp planning algorithms using $Q_{SM}$ as metric are formulated in \prettyref{sec:planning}. Finally, we compare $Q_{SM}$ with conventional metric $Q_1$ in \prettyref{sec:results}.
\section{\label{sec:related}Related Work}
We review related work in grasp quality metric, material and fracture modeling, and grasp planning.

\TE{Grasp Quality Metric:} Although some grasp planners only consider external wrenches along some certain directions \cite{mahler2016energy}, more prominent characterization of robust grasp requires resilience to external wrenches along all directions, which is known as force closure \cite{1087483}. However, infinitely many grasps can have force closure, of which good grasps are characterized by different quality metrics \cite{roa2015grasp}. Most of the metrics $Q$ are designed such that $Q>0$ implies force closure. A closely related metric to our $Q_{SM}$ is $Q_{1,\infty}$ \cite{219918}, which measures the maximal radius of origin-centered wrench-space circle contained in the convex set of resistible wrenches, under contact force magnitude constraints.

\TE{Material and Fracture Modeling:} Real world solid objects will undergo either brittle or ductile fractures depending on their materials. But modeling them are both theoretically and computationally difficult \cite{francfort1998revisiting}. Fortunately, for grasp planning, we do not need to model the deformation of objects after fractures, but only need to detect where fractures might happen. In this case, the theory of linear elasticity suffices, which is efficient to compute using finite element method (FEM) \cite{hughes2012finite} or boundary element method (BEM) \cite{cruse1977numerical}. We use BEM as our computational tool because it only requires a surface triangle mesh which is more amenable to robot grasp applications.

\TE{Grasp Planning:} Given a grasp quality metric, an (asymptotically) optimal grasp planning algorithm finds a grasp that maximizes the grasp quality metric. Early algorithms \cite{6563868} use sampling-based methods for planning. These algorithms are very general and they are agnostic to the type of grasp quality metrics. However, more efficient algorithms such as \cite{Dai2018,schulman2017grasping} can be designed if quality metrics have certain properties. Recent work \cite{inproceedingsDexNetTwo} uses a set of precomputed quality metrics to train a grasp quality function represented by a deep neural network and then uses the function as reward to optimize a grasp planner via deep reinforcement learning.
\section{\label{sec:problem}Problem Statement}
In this section, we formulate the problem of stress-minimizing grasp planning. Throughout the paper, we keep a 3D target object in its reference space, where the origin coincides with its center-of-mass. The object takes up a volume which is a closed subset $\Omega\subset R^3$. In addition, the object is under an external 6D wrench $\ww$ and a set of $N$ external contact forces $\ff_{1,\cdots,N}$ at contact points $\xx_{1,\cdots,N}$ with unit contact normals $\nn_{1,\cdots,N}$. If a grasp is valid, we have the following wrench balance condition:
\begin{equation}
\begin{aligned}
\label{eq:force_balance}
\ww=-\sum_{i=1}^N \TWOC{\ff_i}{\xx_i\times \ff_i}\quad\ST\;\|(\id-\nn_i\nn_i^T)\ff_i\|\leq\theta \nn_i^T\ff_i,
\end{aligned}
\end{equation}
where $\theta$ is the frictional coefficient. To compare the quality of two different grasps, a well-known method is to compare their $Q_1$ metric \cite{219918}, which is the maximal radius of the origin-centered inscribed sphere in the convex hull of all possible resistible external wrenches when the magnitude of $\ff_i$ is bounded. Mathematically, this is:
\begin{align*}
Q_1=&\fmax{}\;r\quad\ST\{\ww|\ww^T\Metric\ww\leq r^2\}\subseteq \\
&\{\ww|\exists\ff_{1,\cdots,N},\;\ST\prettyref{eq:force_balance},\sum_{i=1}^N\|\ff_i\|^2\leq1\},
\end{align*}
where $\Metric$ is the $6\times6$ positive semi-definite metric tensor in the wrench space.

However, this conventional formulation assumes that the object will never be broken however large the external forces are. To relax this condition, we have to make use of the numerical models of brittle fracture, e.g. \cite{francfort1998revisiting}. In these formulations, we assume that the object is made of homogeneous isotropic elastic material with $\lambda,\mu$ being its Lam\'e material parameters. When under external force fields $\GG(\xx): R^3\to R^3$, an infinitesimal displacement $\uu(\xx): R^3\to R^3$ and an stress field $\stress(\xx): R^3\to R^{3\times3}$ will occur $\forall \xx\in\Omega$. $\uu(\xx),\stress(\xx)$ can be computed from $\GG(\xx)$ using the force balance condition:
\begin{equation}
\begin{aligned}
\label{eq:elasticity}
\forall \xx\in\Omega:\quad
&\deform=(\nabla \uu+\nabla \uu^T)/2 \\
&\stress=2\mu\deform + \lambda\TR(\deform)\id    \\
&\nabla\cdot\stress+\GG=0   \\
\forall \xx\in\partial\Omega:\quad
&\nn(\xx)\cdot\stress+\sum_{i=1}^N\delta(\xx-\xx_i)\ff_i=0,
\end{aligned}
\end{equation}
where we assume the boundary of $\Omega$ is almost everywhere smooth with unit outward normal defined as $\nn(\xx)$ and $\delta$ is the Dirac's delta operator. Classical theory of brittle fracture further assumes that there exists a tensile stress $\stress_{max}$ and brittle fractures will not happen if the following condition holds:
\begin{align}
\label{eq:safe_condition}
\forall \|\dd\|=1, \xx\in\Omega:\; -\stress_{max}\leq\dd^T\stress(\xx)\dd\leq\stress_{max},
\end{align}
note that the stress tensor must be symmetric so that its singular values coincide with its eigenvalues. Given this condition, there are two goals of this paper:
\begin{itemize}
    \item Propose a grasp quality metric $Q_{SM}$ that measures the quality of grasps with \prettyref{eq:safe_condition} as precondition (\prettyref{sec:metric}).
    \item Propose grasp planning algorithms that generate grasps maximizing $Q_{SM}$ (\prettyref{sec:planning}).
\end{itemize}

\begin{figure*}[ht]
\centering
\includegraphics[width=0.98\textwidth]{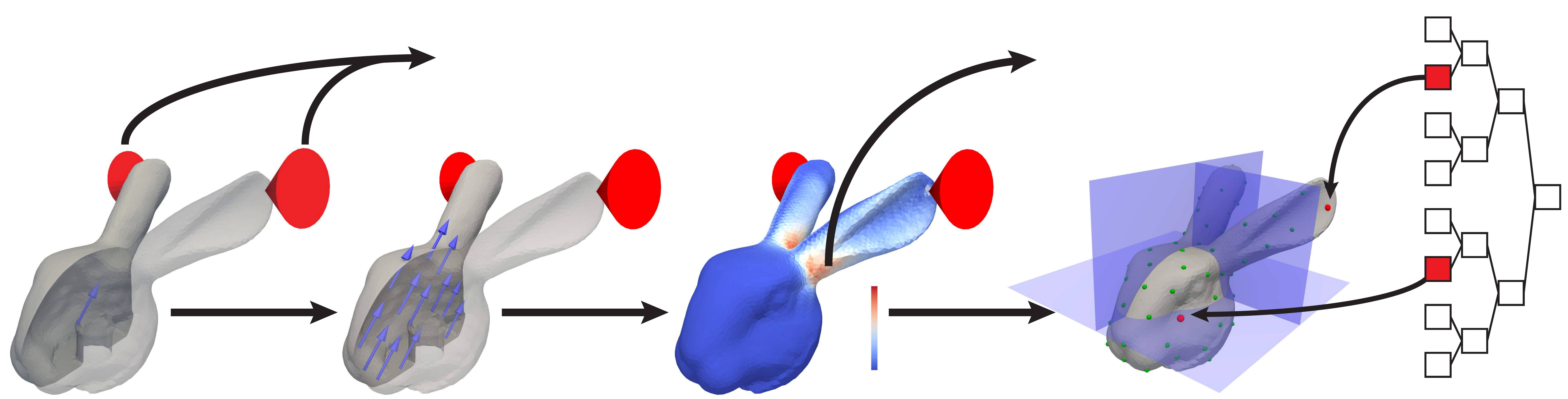}
\put(-480,60){(a)}
\put(-375,60){(b)}
\put(-270,60){(c)}
\put(-165,60){(d)}
\put(-20 ,5 ){(e)}
\put(-440,15){\prettyref{eq:gg_from_w}}
\put(-335,15){\prettyref{eq:elasticity}}
\put(-212,15){\prettyref{pb:planning_relaxed}}
\put(-350,105){$Q_1: \sum_{i=1}^N\|\ff_i\|^2\leq1$}
\put(-163,105){$Q_{SM}: |\dd^T\stress\dd|\leq\stress_{max}$}
\vspace{-0px}
\caption{\label{fig:pipeline} Illustration of our method where the target object is a bunny head and the two-point grasp has contact points on the ears of bunny (frictional cones in red). (a): When the bunny head is under external wrench (blue arrow), $Q_1$ metric assumes that the wrench is applied on the center-of-mass. $Q_1$ further assumes that the magnitude of force is bounded. (b): Our $Q_{SM}$ metric assumes that the external wrench is applied as a body force field (blue arrows). (c): We use BEM to solve for a surface stress field (color coded on surface, highest stress around connections between the ears and the head). $Q_{SM}$ assumes that the stress along any direction $\E{d}$ is smaller than tensile stress $\stress_{max}$. (d): When performing grasp planning, we first construct a KD-tree (transparent blue planes) for the set of $N$ potential contact points. (e): We select $C$ contact points (red points) by descending the tree. This is equivalent to a mixed-integer programming problem with $\E{log}_2(N)C$ binary decision variables.}
\vspace{-10px}
\end{figure*}
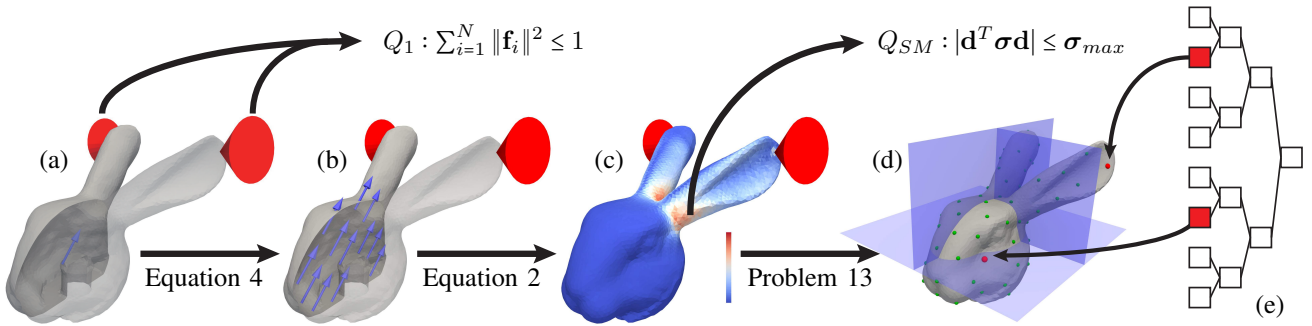
\section{\label{sec:metric}The Stress-Minimizing Metric $Q_{SM}$}
Our construction is illustrated in \prettyref{fig:pipeline}. The basic idea behind the construction of $Q_{SM}$ is very similar with that of the $Q_1$ metric. Intuitively, we first define a convex subset $\WW$ of the 6D wrench space, which contains resistible wrenches that does not violate \prettyref{eq:safe_condition}. We then define $Q_{SM}$ as the maximal radius of the origin-centered sphere contained in $\WW$.

\subsection{Definition of $\WW$ and $Q_{SM}$}
Given a certain wrench $\ww$, we need to determine whether $\ww\in\WW$. This can be done by first computing $\stress$ and then testing whether \prettyref{eq:safe_condition} holds. However, $\stress$ is computed from $\GG$ but not $\ww$, so that we need to find a relationship between $\GG$ and $\ww$. In other words, we need to find a body force distribution such that the net effect of $\GG$ is equivalent to applying $\ww$ on the center-of-mass. Obviously, infinitely many $\GG$ will satisfy this relationship and different choices of $\GG$ will lead to different variants of $Q_{SM}$ metrics. In this paper, we propose to choose $\GG$ as a linear function in $\xx$. The most important reason behind this choice is that the computation of $\stress$ can be accomplished using BEM if $\GG$ is a harmonic function of $\xx$. Under this choice, we have: $\GG(\xx)=\GG_0+\nabla\GG\xx$, where $\GG_0$ is the constant term, $\nabla\GG$ is the constant spatial derivative tensor. Clearly $\GG(\xx)$ has 12 degrees of freedom and we can solve for $\GG_0$ and $\nabla\GG$ to equate the effect of $\GG$ and $\ww$ as follows:
\begin{align*}
\ww&=\int_\Omega\TWOC{\GG}{\xx\times\GG}d\xx=
\TWOC{|\Omega|\GG_0}{\TT\THREEC{{[\nabla\GG]}_x}{{[\nabla\GG]}_y}{{[\nabla\GG]}_z}}    \\
\TT&\triangleq\int_\Omega\TTMAT d\xx,
\end{align*}
where ${[\nabla\GG]}_{x,y,z}$ are the first, second, and third column of $\nabla\GG$, respectively. However, there are 9 degrees of freedom in $\GG_\xx$ but only 3 constraints, so that we have to solve for $\GG_\xx$ in a least square sense:
\begin{align*}
\GG_\xx=\argmin{\GG_\xx}\;\int_\Omega\|\GG_\xx\xx\|^2d\xx\quad
\ST\;\TT\THREEC{{[\nabla\GG]}_x}{{[\nabla\GG]}_y}{{[\nabla\GG]}_z}=\THREEC{\ww_4}{\ww_5}{\ww_6},
\end{align*}
the solution of which can be computed analytically. In summary, we have:
\begin{align}
\label{eq:gg_from_w}
\FOURC{\GG_0}{{[\nabla\GG]}_x}{{[\nabla\GG]}_y}{{[\nabla\GG]}_z}&=
\MTT{\frac{\id}{|\Omega|}}{}{}{\MM^{-1}\TT^T\left[\TT\MM^{-1}\TT^T\right]^{-1}}\ww
\end{align}
\begin{align*}
\MM&\triangleq\left[\int_\Omega\MMMAT d\xx\right]\otimes\id,
\end{align*}
where $\otimes$ denotes Kronecker product. The matrices $\TT,\MM$ are constants and can be precomputed from the shape of the target object, or more specifically, from the inertia tensor. Given these definitions, we can now define $\WW$ as follows:
\begin{align}
\WW=\{\ww|\exists\GG,\ff_{1,\cdots,N},\uu,\deform,\stress,\;\ST\;
\prettyref{eq:force_balance},\ref{eq:elasticity},\ref{eq:safe_condition},\ref{eq:gg_from_w}\}.
\end{align}
Finally, we are ready to give a mathematical definition of $Q_{SM}$ as the following optimization problem:
\begin{align*}
Q_{SM}=\fmax{}\;r\quad
\ST\;\{\ww|\ww^T\Metric\ww\leq r^2\}\subseteq\WW.
\end{align*}

From the mathematical definition of $Q_{SM}$, we immediately have the following properties of $\WW$:
\begin{lemma}
$\WW$ is a convex set.
\end{lemma}
\begin{proof}
\prettyref{eq:force_balance} is a set of quadratic cone constraints, which defines a convex set. \prettyref{eq:elasticity} is a set of infinite-dimensional linear constraints, which defines a convex set. \prettyref{eq:safe_condition} is an infinite-dimensional PSD-cone constraint, which defines a convex set. Finally, \prettyref{eq:gg_from_w} is a linear constraint, which also defines a convex set. As the intersection of convex sets, $\WW$ is convex.
\end{proof}
\begin{lemma}
\label{Lem:compact}
$\WW$ is a compact set so that $Q_{SM}$ is finite.
\end{lemma}
\begin{proof}
For any $\ww\neq0$, $\stress$ that satisfies \prettyref{eq:force_balance}, \prettyref{eq:elasticity}, and \prettyref{eq:gg_from_w} cannot be uniformed zero. In other words, there exists $\dd$ and $\xx$ such that $\dd^T\uu(\xx)\dd>\epsilon>0$. If we multiply $\ww$ by $\alpha>\stress_{max}/\epsilon$, \prettyref{eq:safe_condition} will be violated so that $\alpha\ww\notin\WW$. Therefore, $\WW$ is bounded and is obviously closed, so that $\WW$ is compact and $Q_{SM}$ is finite.
\end{proof}

In addition, the following property of $Q_{SM}$ is obvious:
\begin{lemma}
$Q_{SM}>0$ implies force closure.
\end{lemma}
And the following property has been proved in \cite{schulman2017grasping} for $Q_1$ and also holds for $Q_{SM}$ by a similar argument:
\begin{lemma}
\label{Lem:opt_formulation}
$Q_{SM}=\fmin{\dd,\|\dd\|=1}\;\fmax{\ww\in\sqrt{\Metric}\WW}\;\ww^T\dd$.
\end{lemma}

\subsection{Discretization of $Q_{SM}$}
The computation of exact $Q_{SM}$ is impossible because it involves infinite dimensional tensor fields: $\stress,\deform$, so that we have to discretize them using conventional techniques such as FEM \cite{hughes2012finite} or BEM \cite{cruse1977numerical}. In comparison, FEM is mathematically simpler but requires a volumetric mesh of the target object, while BEM only requires a surface triangle mesh. We provide the detailed derivation of BEM in \prettyref{appen:BEM} and summarize the main results here. Our BEM implementation approximates the stress field $\stress(\xx)$ to be piecewise constant on each triangular patch of the surface. Assuming that the target object has $K$ surface triangles whose centroids are: $\xx_{1,\cdots,K}$, we have $K$ different stress values:
\begin{align}
\label{eq:BEM}
\THREEC{\stress_x(\xx_j)}{\stress_y(\xx_j)}{\stress_z(\xx_j)}=
\BEMA_j\FOURC{\GG_0}{{[\nabla\GG]}_x}{{[\nabla\GG]}_y}{{[\nabla\GG]}_z}+
\BEMB_j\THREEC{\ff_1}{\vdots}{\ff_N}\quad\forall j=1,\cdots,K,
\end{align}
where $\BEMA,\BEMB$ are dense coefficient matrices defined from BEM discretization, the definitions of which can be found from \prettyref{appen:BEM}. Note that computing the coefficients of these two matrices are very computationally costly, where a naive implementation of BEM requires $\mathcal{O}(K^3)$ operations and acceleration techniques such as the H-matrix \cite{hackbusch1999sparse} can reduce this cost to $\mathcal{O}(K\E{log}^2K)$ operations. However, these two matrices are constant and can be precomputed for a given target object shape, so that the the cost of BEM computation is not a part of grasp planning. After discretization, we arrive at the finite-dimensional version of fracture condition:
\begin{align}
\label{eq:safe_condition_discrete}
\forall \|\dd\|=1, j=1,\cdots,K:\; -\stress_{max}\leq\dd^T\stress(\xx_j)\dd\leq\stress_{max}.
\end{align}
finite-dimensional version of $\bar{\WW}$:
\begin{align*}
\bar{\WW}\triangleq\{\ww|\exists\GG,\ff_{1,\cdots,N},\uu,\deform,\stress,\;\ST\;
\prettyref{eq:force_balance},\ref{eq:BEM},\ref{eq:safe_condition_discrete},\ref{eq:gg_from_w}\},
\end{align*}
and finite-dimensional version of $\bar{Q}_{SM}$ defined as:
\begin{align*}
\bar{Q}_{SM}=\argmax{}\;r\quad
\ST\;\{\ww|\ww^T\Metric\ww\leq r^2\}\subseteq\bar{\WW}.
\end{align*}
All the properties of the infinite-dimensional $\WW$ and $Q_{SM}$ hold for the finite-dimensional version $\bar{\WW}$ and $\bar{Q}_{SM}$ by a similar argument. 
\subsection{Computation of $\bar{Q}_{SM}$}
Even after discretization, computing $\bar{Q}_{SM}$ is costly and non-trivial. According to \prettyref{Lem:opt_formulation}, the equivalent optimization problem for $\bar{Q}_{SM}$ is:
\begin{align*}
\bar{Q}_{SM}=\fmin{\dd,\|\dd\|=1}\;\fmax{\ww\in\sqrt{\Metric}\bar{\WW}}\;\ww^T\dd,
\end{align*}
which is non-convex optimization, so that direct optimization does not give the global optimum. In this section, we modify three existing algorithms to (approximately) compute $\bar{Q}_{SM}$. Two discrete algorithms have been proposed in \cite{zheng2012efficient,schulman2017grasping} to approximate $Q_1$. Due to the close relationships between $Q_1$ and $\bar{Q}_{SM}$, we show that these two algorithms can be modified to compute $\bar{Q}_{SM}$.

In \cite{schulman2017grasping}, the space of unit vectors is discretized into a finite set of $D$ directions: $\dd_{1,\cdots,D}$. As a result, we can compute an upper bound for $\bar{Q}_{SM}$ as:
\begin{align*}
\bar{Q}_{SM}\leq\fmin{j=1,\cdots,D}\;\fmax{\ww\in\sqrt{\Metric}\bar{\WW}}\;\ww^T\dd_j.
\end{align*}
We can make this upper bound arbitrarily tight by increasing $D$. In another algorithm \cite{zheng2012efficient}, a convex polytope $\CC\subseteq\bar{\WW}$ is maintained using H-representation \cite{grunbaum1969convex} and we can compute a lower bound for $\bar{Q}_{SM}$ as:
\begin{align}
\label{pb:lower_bound}
\bar{Q}_{SM}\geq\fmin{\dd,\|\dd\|=1}\;\fmax{\ww\in\sqrt{\Metric}\CC}\;\ww^T\dd.
\end{align}
The global optimum of \prettyref{pb:lower_bound} is easy to compute from an H-representation of $\CC$ by computing the distance between the origin and each face of $\CC$. This lower bound can be iteratively tightened by first computing the blocking face normal $\dd$ of $\CC$ and then expanding $\CC$ via:
\begin{align*}
\CC\gets\hull{\CC\cup\left\{\argmax{\ww\in\sqrt{\Metric}\bar{\WW}}\;\ww^T\dd\right\}}.
\end{align*}
These two algorithms can be realized if we can find the supporting point of $\sqrt{\Metric}\bar{\WW}$, which amounts to the following conic programming problem:
\begin{equation}
\begin{aligned}
\label{pb:support}
\argmax{\ww,\ff_i,\stress(\xx_j)}\;&\ww^T\sqrt{\Metric}\dd  \\
\ST\;&\prettyref{eq:force_balance},\ref{eq:BEM},\ref{eq:gg_from_w}  \\
&-\stress_{max}\id\preceq\stress(\xx_j)\preceq\stress_{max}\id
\quad\forall j=1,\cdots,K.
\end{aligned}
\end{equation}
The conic programming reformulation in \prettyref{pb:support} can be solved using the interior point method \cite{andersen2000mosek}. Given this solution procedure, we summarize the modified version of \cite{schulman2017grasping} in \prettyref{Alg:UpperBound} and modified version of \cite{zheng2012efficient} in \prettyref{Alg:LowerBound}. Note that \prettyref{Alg:LowerBound} is advantageous over \prettyref{Alg:UpperBound} in that it can approximate $\bar{Q}_{SM}$ up to arbitrary precision $\epsilon$, so we always use \prettyref{Alg:LowerBound} in the rest of the paper. 

Compared with $Q_1$ metric, a major limitation of using $\bar{Q}_{SM}$ metric is that the computational cost is much higher. Note that the computational cost of solving \prettyref{pb:support} is at least linear in $K$ and can be superlinear depending on the type of conic programming solver used. This $K$ is the number of surface triangles on the target object, which can easily reach several thousands. Fortunately, we can drastically reduce this cost by using a progressive approach.

\subsection{Performance Optimization}
The naive execution of \prettyref{Alg:LowerBound} can be prohibitively slow due to the repeated solve of \prettyref{pb:support}. The conic programming problem has $K$ PSD-cone constraints with $K$ being several thousands. Solving \prettyref{pb:support} using interior point method \cite{andersen2000mosek} involves repeated solving a sparse linear system with size proportional to $K$. We propose a method that can greatly improve the performance when solving \prettyref{pb:support}. Our idea is that when the global optimum of \prettyref{pb:support} is reached, more than $99\%$ of the $K$ PSD-cone constraints are inactive, so that removing these constraints do not alter the solution. This idea is inspired by \cite{Zhou:2013:WSA:2461912.2461967} which shows that, empirically, maximal stress only happens on a few sparse points on the surface of the target object. However, we do not know the active constraints as a prior. Therefore, we propose to progressively detect these active constraints.

To do so, we first select a subset $\KK\subset\{1,\cdots,K\}$ such that $|\KK|\ll K$ and $\{\stress(\xx_i)|i\in\KK\}$ are the stresses that are most likely to be violated. To select this set $\KK$, we use a precomputation step and solve \prettyref{pb:support} for $S$ times using random $\dd$, and record which PSD-cones are active. For each PSD-cone, we maintain how many times they become active during the $S$ solves of \prettyref{pb:support}. We then select the most frequent $|\KK|$ PSD-cones to form $\KK$. After selecting $\KK$, we maintain an active set $\SSS$ which initializes to $\KK$ and we solve \prettyref{pb:support} using constraints only in $\SSS$, which is denoted by:
\begin{equation}
\begin{aligned}
\label{pb:support_cons}
\argmax{\ww,\ff_i,\stress(\xx_j)}\;&\ww^T\sqrt{\Metric}\dd  \\
\ST\;&\prettyref{eq:force_balance},\ref{eq:BEM},\ref{eq:gg_from_w}  \\
&-\stress_{max}\id\preceq\stress(\xx_j)\preceq\stress_{max}\id
\quad\forall j\in\SSS.
\end{aligned}
\end{equation}
After we solve for the global optimum of \prettyref{pb:support_cons}, we check the stress of remaining constraint points and we pick the most violated constraint:
\begin{align}
\label{eq:most_violated}
j^*=\argmax{j\in\{1,\cdots,K\}/\SSS}\sqrt{\|\stress(\xx_j)\stress(\xx_j)\|_2}.
\end{align}
If we have $\sqrt{\|\stress(\xx_{j^*})\stress(\xx_{j^*})\|_2}<\stress_{max}$, then \prettyref{pb:support_cons} and \prettyref{pb:support} will return the same solution. Otherwise, we add $j^*$ to $\SSS$. This method is summarized in \prettyref{Alg:progressive} and is guaranteed to return the same global optimum of \prettyref{pb:support}. In practice, \prettyref{Alg:progressive} is orders of magnitude more efficient than throwing all constraints to the interior point method at once.

\section{\label{sec:planning}Grasp Planning Under The SM Metric}
Built on top of the computational procedure of $\bar{Q}_{SM}$, we can solve the optimal grasp point selection problem. Given a set of $N$ potential grasp points sampled on $\partial\Omega$, we want to select $C$ points. Mathematically, we want to solve the following mixed-integer optimization problem using branch-and-bound (BB) algorithm \cite{clausen1999branch}:
\begin{equation}
\begin{aligned}
\label{pb:planning}
\argmax{z_i}\;&\bar{Q}_{SM}   \\
    \ST\;&z_i\in\{0,1\}\quad\forall i=1,\cdots,N \\
              &\|\ff_i\|^2\leq z_iM\wedge
              \sum_{i=1}^N z_i\leq C,
\end{aligned}
\end{equation}
where $M$ is the big-M constant that can be arbitrarily large. Note that the algorithm to solve \prettyref{pb:planning} can be different depending on what algorithm we use to (approximately) compute $\bar{Q}_{SM}$. If approximate \prettyref{Alg:UpperBound} is used, then the problem can be efficiently solved using the sub-modular coverage algorithm \cite{schulman2017grasping}. However, in order to highlight the advantage of our new metric, we would like to compute $\bar{Q}_{SM}$ accurately using \prettyref{Alg:LowerBound}. In \prettyref{sec:opt_planning}, we show that \prettyref{pb:planning} can be solved alot more efficiently using a series of reformulations, without changing the global optimum.

\subsection{Optimized BB Algorithm \label{sec:opt_planning}}
Our key innovation is via the following reformulation:
\begin{equation}
\begin{aligned}
\label{pb:planning_relaxed}
\argmax{z_i^j}\;&\bar{Q}_{SM}   \\
    \ST\;&z_i^j\in\textcolor{red}{[0,1]}\quad\forall i=1,\cdots,N\; j=1,\cdots,C \\
              &\|\ff_i\|^2\leq \sum_{j=1}^Cz_i^jM\wedge
              \{z_{1,\cdots,N}^j\}\in\textcolor{red}{\mathcal{SOS}_1}   \\
              &z_i^j\leq z_i^{j+1}\quad\forall j=1,\cdots,C-1,
\end{aligned}
\end{equation}
where $\mathcal{SOS}_1$ is the special-ordered-set-of-type-1 \cite{vielma2011modeling}, which constraints that only one number in a set can take non-zero value. According to \cite{vielma2011modeling}, we know that \prettyref{pb:planning} requires $N$ binary variables while \prettyref{pb:planning_relaxed} requires $\E{log}_2(N)C$ binary variables, which is much fewer as $C\ll N$. Intuitively, \prettyref{pb:planning_relaxed} build a binary bounding volume hierarchy for the set of $N$ contact points and introduce one binary decision variable for each internal level of the tree to select whether the left or the right child is selected. After a leaf node is reached, a contact point is selected. Finally, the last constraint in \prettyref{pb:planning_relaxed}, $z_i^j\leq z_i^{j+1}$, reflects the order-independence of contact points.

Unfortunately, no optimization tools can solve \prettyref{pb:planning_relaxed} in an off-the-shelf manner due to the special objective function, so that we develope a special implementation of BB outlined in \prettyref{Alg:BB}. In this algorithm, our binary bounding volume hierarchy is a KD-tree, as illustrated in \prettyref{fig:pipeline}de. The most important component for the efficiency of BB algorithm is the problem relaxation (\prettyref{ln:relax}) as follows:
\begin{align}
\label{pb:planning_relaxed_relaxation}
\bar{Q}_{SM}^{curr}=\fmax{\bar{Q}_{SM}}\; &\bar{Q}_{SM}\quad\ST\;\|\ff_i\|^2=0\quad\forall i\notin\mathcal{S},
\end{align}
which can be solving using \prettyref{Alg:LowerBound} by excluding the contact points in $\mathcal{S}$. We summarize our main results below:
\begin{lemma}
\prettyref{Alg:BB} returns the global optimum of \prettyref{pb:planning_relaxed}, which is also the global optimum of \prettyref{pb:planning} by setting: $z_i=\sum_{j=1}^C z_i^j$.
\end{lemma}
Finally, note that all the results in this section hold if we replace $Q_{SM}$ with $Q_1$.
\section{\label{sec:results}Evaluations}
We implement our algorithms for computing $Q_{SM}$ and perform grasp planning using C++. The accuracy of BEM heavily relies on the quality of the surface triangle mesh, so that we first optimize the mesh quality using CGAL \cite{cgal:rty-m3-19a}. We implement the BEM using kernel independent numerical integration scheme \cite{journals/aes/FialaR14}. The most computationally costly step in BEM is the inversion of system matrices, for which we use LU-factorization accelerated by H-matrices \cite{hackbusch1999sparse}. Finally, we use CGAL \cite{cgal:hs-chdt3-19a} to construct convex hulls with exact arithmetics. All the experiments are performed on a single desktop machine with two Xeon E5-2697 CPU and 256Gb memory. 

\TE{Parameter Choices:} Computing $Q_{SM}$ requires more parameters than $Q_1$. Specifically, there are three additional variables: tensile stress $\stress_{max}$ and Lam\'e material parameters: $\mu,\lambda$. However, if we transform $\mu,\lambda$ to an equivalent set of parameters: Young's modulus $E$ and Poisson ratio $\nu$ \cite{hughes2012finite}, it is obvious to show that $Q_{SM}$ is proportional to $\stress_{max}$ and inversely proportional to $E$. Since the absolute value of a grasp metric is meaningless for grasp planning and only the relative value matters, we can always set $\stress_{max}=E=1$ and choose only $\nu$ according to the material type of the target object, and then setting:
\begin{align*}
\mu=\frac{1}{2(1+\nu)}\quad\lambda=\frac{\nu}{(1+\nu)(1-2\nu)}.
\end{align*}
We have $\nu=0.33$ for copper and $\nu=0.499$ for rubber. In all experiments, we set $\nu=0.33$. Finally, when running \prettyref{Alg:LowerBound}, we set $\epsilon=0.001$.

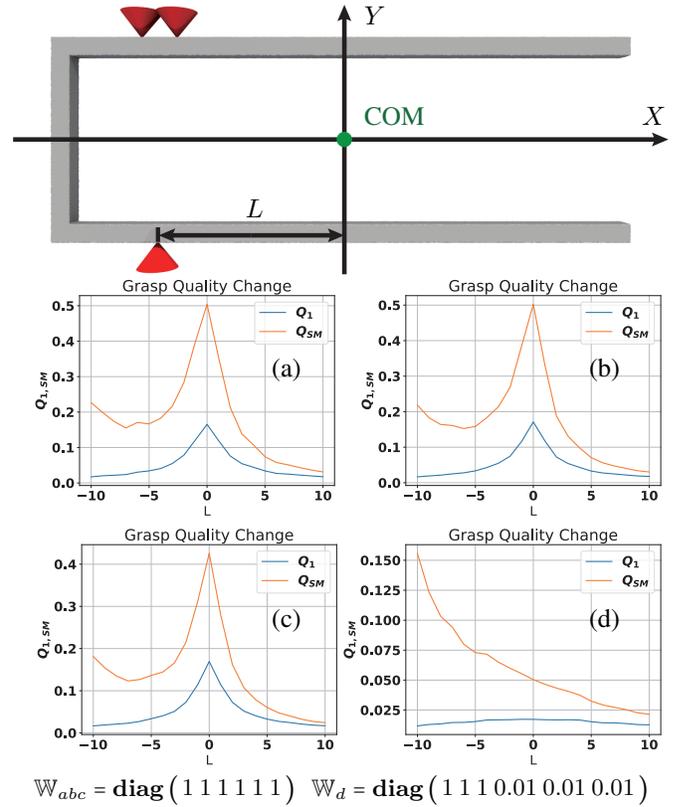
\begin{figure}[ht]
\centering
\includegraphics[width=0.49\textwidth]{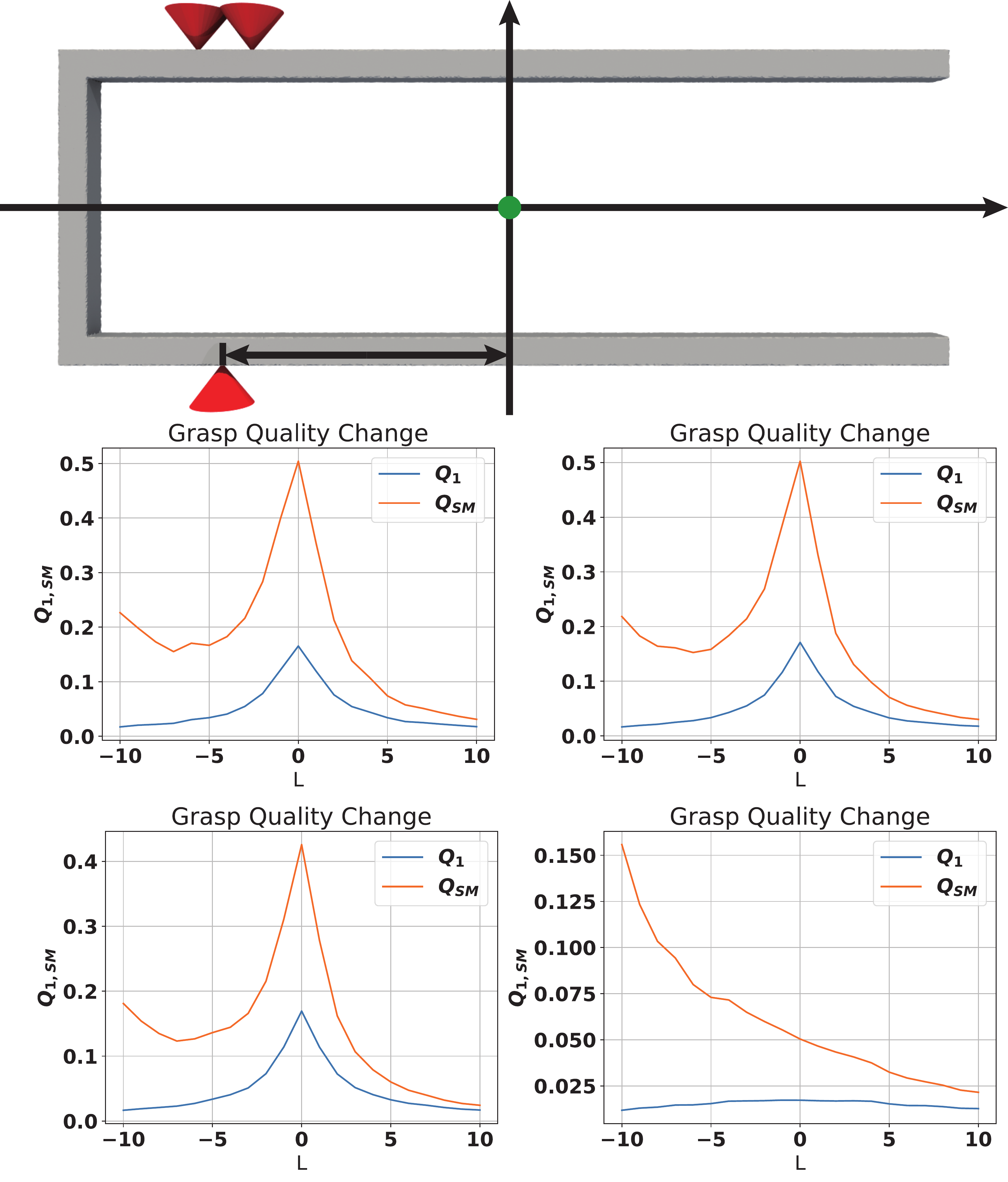}
\put(-160,210){$L$}
\put(-10 ,245){$X$}
\put(-115,283){$Y$}
\put(-115,245){\textcolor{darkgreen}{COM}}
\put(-150,150){(a)}
\put(-30 ,150){(b)}
\put(-150,55 ){(c)}
\put(-30 ,55 ){(d)}
\put(-240,-9){\small$\Metric_{abc}=\E{diag}\SIX{1}{1}{1}{1}{1}{1}$}
\put(-135,-9){\small$\Metric_{d}=\E{diag}\SIX{1}{1}{1}{0.01}{0.01}{0.01}$}
\vspace{-0px}
\caption{\label{fig:shapeAwareness} The target object is a U-shaped tuning fork, where $Q_{SM}$ is aware of shapes while $Q_1$ is not. We test a 3-point grasp (frictional cones in red) where the distance to the center of mass (green) is $L$, we plot the change of $Q_{1,SM}$ against $L$ under 4 different conditions. In (a,b,c), we weight forces and torques equally with $\Metric_{abc}$ and use meshes of different resolutions, with $K=5730$ in (a), $K=24204$ in (b), and $K=94398$ in (c). In (d), we use a lower weight for torques with $\Metric_{d}$ and use $K=24204$.}
\vspace{-0px}
\end{figure}
\TE{Shape-Awareness:} The most remarkable advantage of $Q_{SM}$ over $Q_1$ is shape awareness. In \prettyref{fig:shapeAwareness}, our target object is a U-shaped tuning fork and we use a 3-point grasp. The shape of the tuning fork is asymmetric along the X-axis, and according to \prettyref{fig:shapeAwareness}abc, $Q_1$ is not aware of the asymmetry, while $Q_{SM}$ correctly reflects the fact that grasping the leftmost point is better than grasping the mid-left point because it is less likely to break the object. However, the best grasp under $Q_1$ and $Q_{SM}$ are the same, i.e., grasping the centroid point. If we change the metric $\Metric$ and emphasize force resistance over torque resistance, then the difference between $Q_1$ and $Q_{SM}$ is more advocated.

\TE{Robustness to Mesh Resolution:} The change of $Q_{SM}$ is not sensitive to the resolution of surface meshes as shown in \prettyref{fig:shapeAwareness}abc, which makes $Q_{SM}$ robust to target objects discretized using small, low-resolution meshes. As we increase $K$ from $5730$ to $24204$ and finally to $94398$, the change of $Q_{SM}$ against $L$ is almost intact, with very small fluctuations around $L=-5$.

\begin{figure}[ht]
\centering
\includegraphics[width=0.4\textwidth]{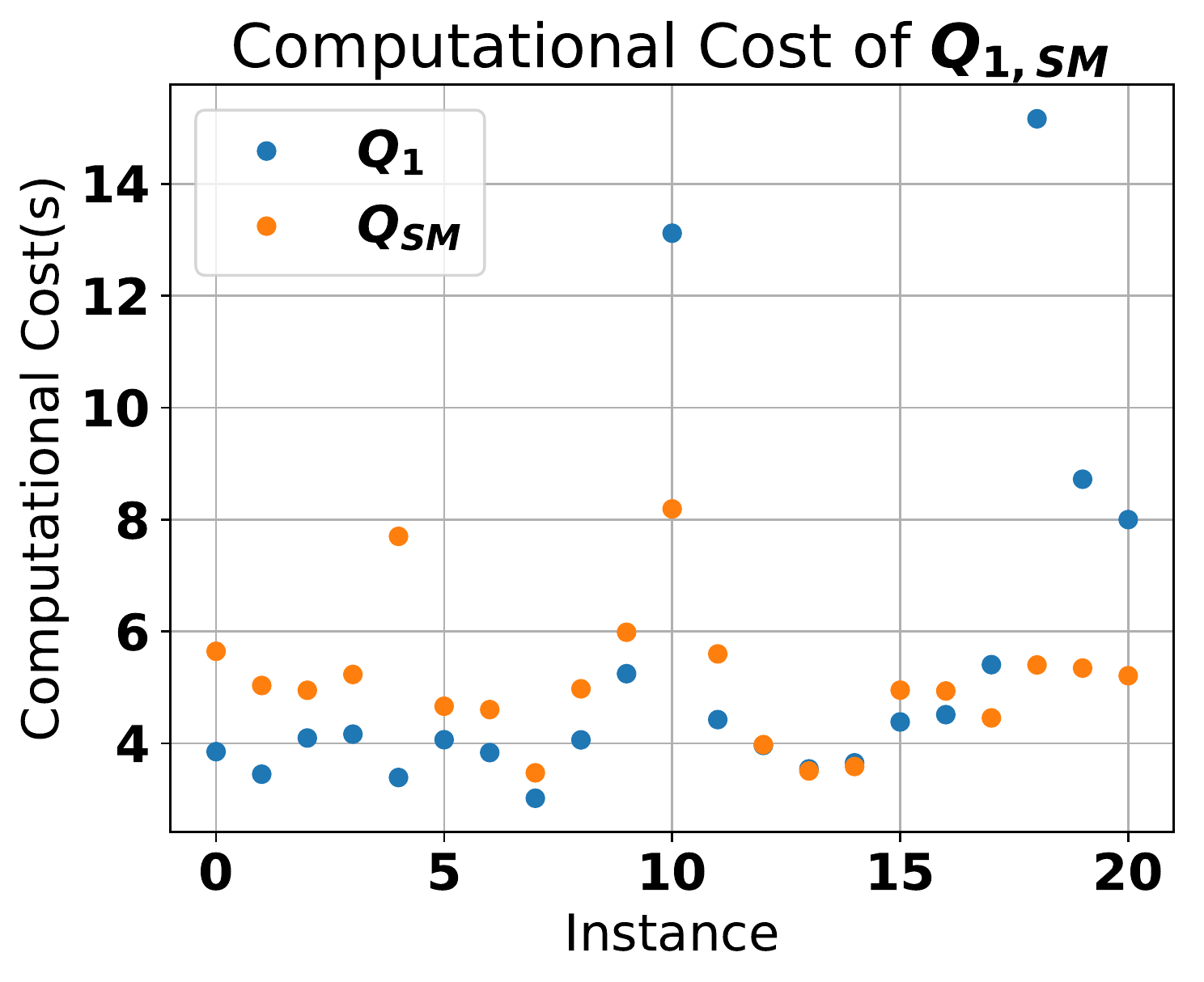}
\vspace{-0px}
\caption{\label{fig:cost} We compare the computational cost of computing $Q_1$ and $Q_{SM}$ for 20 random target objects and grasps.}
\vspace{-0px}
\end{figure}
\TE{Computational Cost:} $Q_{SM}$ does incur a higher computational cost than $Q_1$. The most computational cost lies in the assembly of matrices $\BEMA,\BEMB$ which involves the direct factorization of a large dense matrix. But this assembly is precomputation and required only once for each target object before grasp planning. In \prettyref{fig:shapeAwareness}abc, this step takes $112$s when $K=5730$, $1425$s when $K=24204$, and $3892$s when $K=94398$. After precomputation, the cost of evaluating $Q_{SM}$ and $Q_1$ are very similar, as shown in \prettyref{fig:cost}. This implies that using $Q_{SM}$ does not incur a higher cost in grasp planning. This is largely due to the progressive \prettyref{Alg:progressive}, which greatly reduce the number of constraints in solving \prettyref{pb:support}. Without this method, solving \prettyref{pb:support} is prohibitively costly by requiring the solve of a sparse linear system of size proportional to $K$.

\begin{figure*}[ht]
\centering
\includegraphics[width=0.99\textwidth]{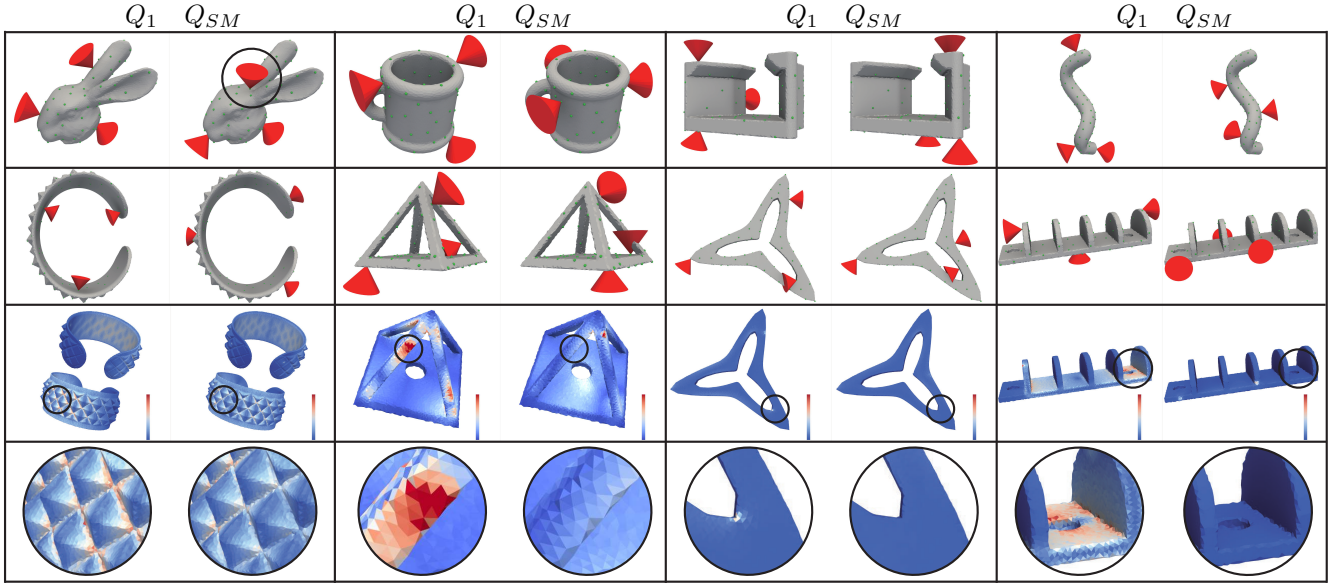}
\put(-455,212){$Q_1\quad Q_{SM}$}
\put(-330,212){$Q_1\quad Q_{SM}$}
\put(-205,212){$Q_1\quad Q_{SM}$}
\put(-80 ,212){$Q_1\quad Q_{SM}$}
\vspace{-0px}
\caption{\label{fig:plan_results} In the first and second row, we show globally optimal grasps for 8 different target objects under both the $Q_{SM}$ and $Q_1$ metric. These grasps are generated by choosing $C=3$ contact points (frictional cones in red) from $N=100$ potential contact points (green). We choose objects for which optimal grasps are very different under the two metrics. In particular, the optimal grasp for the bunny head object is avoiding the ears of bunny under $Q_{SM}$ metric (black circle). In the third and forth row, we plot the maximal stress configuration (color coded) for results corresponding to the second row. Using $Q_{SM}$ can drastically reduce the maximal stress as indicated in the black circle.}
\vspace{-0px}
\end{figure*}
\TE{Grasp Planning:} In \prettyref{fig:plan_results}, we show globally optimal grasps for 8 different target objects under both the $Q_{SM}$ and $Q_1$ metric. To generate these results, we choose $C=3$ contact points from $N=100$ potential contact points by running \prettyref{Alg:BB}. These contact points are generated using Poisson disk sampling. The computational cost of \prettyref{Alg:BB} is $1.7$hr under $Q_{SM}$ and $0.6$hr under $Q_1$ on average. This result is surprising as the cost of computing $Q_{SM}$ is comparable to that of computing $Q_1$. We found that $Q_{SM}$ tends to create more local minima so that BB needs to create a larger search tree under $Q_{SM}$. In the third row of \prettyref{fig:plan_results}, we show the maximal stress configuration in $\WW$ and the corresponding stress configuration under $Q_1$ side-by-side. The advantage of $Q_{SM}$ is quite clear which suppresses the stress to resist the same external wrench. For some target objects, the high stress is concentrated in a very small region and we indicate them using black circles.
\section{\label{sec:conclusion}Conclusion and Limitation}
We present SM metric, which reflects the tendency to break a target object. As a result, a grasp maximizing $Q_{SM}$ will minimize the probability of breaking a fragile object. We show that $Q_{SM}$ can be computed using previous methods and its computational costly can be drastically reduced by progressively detecting the active set. Finally, we show that grasp planning under $Q_{SM}$ can be performed using BB algorithms. Our experiments show that $Q_{SM}$ is aware of geometric fragility while $Q_1$ is not. We also show that using $Q_{SM}$ does not increase computationally cost in grasp planning.

The major limitation of our work is that computing $Q_{SM}$ requires a costly precomputation step for solving the BEM problem. In addition, the BEM problem requires high quality watertight surface meshes of target objects. An avenue of future research is to infer the value of $Q_{SM}$ for given unknown objects using machine learning, as is done in \cite{inproceedingsDexNetTwo}.
\vspace{-5px}
\bibliographystyle{IEEEtranS}
\bibliography{reference}
\begin{algorithm}[ht]
\caption{\label{Alg:UpperBound} Compute upper bound of $\bar{Q}_{SM}$ using \cite{schulman2017grasping}}
\begin{algorithmic}[1]
\State sample directions $\dd_{1,\cdots,D}$ in $\E{SO}(3)$
\For{$i=1,\cdots,D$}
\State Solve \prettyref{pb:support} with $\dd\gets\dd_i$ for $\ww_i$
\EndFor
\State Return $\fmin{i}\{\ww_i^T\sqrt{\Metric}\dd_i\}$
\end{algorithmic}
\end{algorithm}

\begin{algorithm}[ht]
\caption{\label{Alg:LowerBound} Compute lower bound of $\bar{Q}_{SM}$ using \cite{zheng2012efficient}}
\begin{algorithmic}[1]
\State sample initial directions $\dd_{1,\cdots,D}$ in $\E{SO}(3)$
\For{$i=1,\cdots,D$}
\State Solve \prettyref{pb:support} with $\dd\gets\dd_i$ for $\ww_i$.
\EndFor
\State $\CC_0\gets\hull{\ww_{1,\cdots,D}}$
\State Solve \prettyref{pb:lower_bound} with $\CC\gets\CC_0$ for $\bar{Q}_{SM}^0$
\State Store the blocking face normal on $\CC_0$ as $\dd_0$
\While{$k=1,\cdots$}
\State Solve \prettyref{pb:support} with $\dd\gets\dd_{k-1}$ for $\ww_k$
\State $\CC_k\gets\hull{\CC_{k-1}\cup\{\ww_k\}}$
\State Solve \prettyref{pb:lower_bound} with $\CC\gets\CC_k$ for $\bar{Q}_{SM}^k$
\State Store the blocking face normal on $\CC_k$ as $\dd_k$
\If{$|\bar{Q}_{SM}^k-\bar{Q}_{SM}^{k-1}|<\epsilon$}
\State return $\bar{Q}_{SM}^k$
\EndIf
\EndWhile
\end{algorithmic}
\end{algorithm}

\begin{algorithm}[ht]
\caption{\label{Alg:progressive} Progressive solve of \prettyref{pb:support}}
\begin{algorithmic}[1]
\State $\SSS\gets\KK$
\While{$\SSS\neq\{1,\cdots,K\}$}
\State Solve \prettyref{pb:support_cons} for $\ww,\ff_i,\stress(\xx_j)$
\State Pick $j^*$ using \prettyref{eq:most_violated}
\If{$\sqrt{\|\stress(\xx_{j^*})\stress(\xx_{j^*})\|_2}<1$}
\State Return $\ww,\ff_i,\stress(\xx_j)$
\Else 
\State $\SSS\gets\SSS\cup\{j^*\}$
\EndIf
\EndWhile
\State Return $\ww,\ff_i,\stress(\xx_j)$
\end{algorithmic}
\end{algorithm}

\begin{algorithm}[ht]
\caption{\label{Alg:BB} BB algorithm solving \prettyref{pb:planning_relaxed}}
\begin{algorithmic}[1]
\LineComment{Each node of the KD-tree is denoted as $\KDNode$}
\LineComment{Each $\KDNode$ is a set and the root $\KDNode_{r}=\{1,\cdots,N\}$}
\LineComment{Each internal $\KDNode$ has two children $\KDNode_{l}\cup\KDNode_{r}=\KDNode$}
\LineComment{Each $\KDNode$ with only one element is a leaf node}
\State Build a KD-tree for $N$ contact points
\LineComment{Initialize BB search tree as a queue}
\LineComment{Each node in the search tree is a $C$-tuple}
\LineComment{Each element in the tuple is a node $\KDNode$}
\State $\E{queue}\gets\emptyset$
\LineComment{Initialize search from the root}
\State $\E{queue.insert(}<\KDNode_r,\cdots,\KDNode_r>\E{)}$
\LineComment{Record the best solution so far}
\State $\E{best}\gets\emptyset$ and $\bar{Q}_{SM}^{best}\gets\infty$
\While{$\E{queue}\neq\emptyset$}
\LineComment{We use subscript to index KD-tree}
\LineComment{We use superscript to index contact points}
\State $<\KDNode^1,\cdots,\KDNode^C>\gets\E{queue.pop()}$
\State Solve \prettyref{pb:planning_relaxed_relaxation} with $\mathcal{S}\gets\KDNode^1\cup\KDNode^2\cdots\cup\KDNode^C$\label{ln:relax}
\If{$\bar{Q}_{SM}^{curr} \geq \bar{Q}_{SM}^{best}$}
\LineComment{Bound: Stop search early}
\State continue
\Else
\LineComment{Check for the last condition in \prettyref{pb:planning_relaxed}}
\State $\E{valid}\gets True$
\For{$j=1,\cdots,C-1$}
\If{$\forall a\in\KDNode^j,\;b\in\KDNode^{j+1}$ we have $a>b$}
\State $\E{valid}\gets False$
\EndIf
\EndFor
\If{not $\E{valid}$}
\State continue
\EndIf
\LineComment{Branch: Desend the KD-tree}
\State $\E{isLeaf}\gets True$
\For{$j=1,\cdots,C$}
\If{$\KDNode^j$ is not leaf node}
\State $\E{isLeaf}\gets False$
\State $\E{queue.insert(}<\KDNode^1,\cdots,\KDNode_l^j,\cdots,\KDNode^C>\E{)}$
\State $\E{queue.insert(}<\KDNode^1,\cdots,\KDNode_r^j,\cdots,\KDNode^C>\E{)}$
\State break
\EndIf
\EndFor
\LineComment{Found a feasible solution}
\If{$\E{isLeaf}$}
\State $\E{best}\gets<\KDNode^1,\cdots,\KDNode^C>$ and $\bar{Q}_{SM}^{best}\gets 
\bar{Q}_{SM}^{curr}$
\EndIf
\EndIf
\EndWhile
\For{$j=1,\cdots,C$ and $i=1,\cdots,N$}
\If{$i\in\KDNode^j$}
\State $z_i^j\gets1$
\Else
\State $z_i^j\gets0$
\EndIf
\EndFor
\end{algorithmic}
\end{algorithm}
\appendix[\label{appen:BEM}The Boundary Element Method]
In this section, we summarize the boundary element discretization of \prettyref{eq:elasticity} and the definition of $\BEMA_i,\BEMB_i$. In addition, we derive the special form of BEM with our body force and external traction distribution. We follow \cite{steinbach2007numerical} with minor changes. First, we define a set of notations and useful theorems. For any $3\times3$ matrix such as $\stress$, we have:
\begin{align*}
\nabla\cdot\stress=\THREEC{\nabla\cdot\stress_x}{\nabla\cdot\stress_y}{\nabla\cdot\stress_z}.
\end{align*}
For any $3$ vector such as $\uu$, we have:
\begin{align*}
&\Delta\uu=\nabla\cdot(\nabla\uu^T)=\THREEC{\Delta\uu_x}{\Delta\uu_y}{\Delta\uu_z}  \\
&\nabla\cdot\nabla\uu=\nabla\nabla\cdot\uu=\nabla\cdot(\tr{\nabla\uu}\id).
\end{align*}

\subsection{Elastostatic Equation in Operator Form}
We first derive the operator form of the elastostatic problem. By combining the three equations in \prettyref{eq:elasticity}, we have:
\begin{equation}
\begin{aligned}
\label{eq:equilibrium}
0&=\nabla\cdot(\mu\nabla\uu+\mu\nabla\uu^T+\lambda\tr{\nabla\uu}\id)+\GG    \\
 &=-\LL[\uu]+\GG=\nabla\cdot\stress+\GG  \\
\LL[]
 &\triangleq-(\mu+\lambda)\nabla\nabla\cdotp[]-\mu\Delta[].
\end{aligned}
\end{equation}

\subsection{Boundary Integral Equation (BIE)}
Next, we derive BIE via the divergence theorem:
\begin{align*}
 &\int_\Omega \vv^T\GG d\pxx
= \int_\Omega \vv^T\LL[\uu] d\xx
=-\int_\Omega \vv^T\nabla\cdot\stress d\xx \\
=&-\int_\Omega \nabla\cdot(\stress\vv) d\xx + \int_\Omega \tr{\nabla\vv\stress} d\xx    \\
=&-\int_{\partial\Omega} \vv^T\stress\nn d\pxx + \E{Sym}(\uu,\vv),
\end{align*}
where $\E{Sym}(\uu,\vv)$ denotes a symmetric term in $\uu,\vv$ satisfying: $\E{Sym}(\uu,\vv)=\E{Sym}(\vv,\uu)$. We then swap $\uu,\vv$ and subtract the two equations to get:
\begin{align}
\label{eq:BIE}
&\int_\Omega (\uu^T\LL[\vv]-\vv^T\LL[\uu]) d\xx=
\int_{\partial\Omega} (\vv^T\NN[\uu]-\uu^T\NN[\vv]) d\pxx\nonumber  \\
&\NN[]\triangleq \stress[]\nn=
\mu\nabla[]\nn+\mu\nabla[]^T\nn+\lambda\nabla\cdot[]\nn.
\end{align}
\subsection{Fundamental Solution}
If $\vv$ is the fundamental solution centered at $\tilde{\xx}$, which is denoted by $U(\xx-\tilde{\xx})$, then we have the boundary integral equation by plugging $U$ into \prettyref{eq:BIE}:
\begin{align}
\label{eq:BIE_FS}
&\int_\Omega U^T\GG d\xx+\int_{\partial\Omega} (U^T\NN[\uu]-\NN[U]^T\uu) d\pxx=\uu(\tilde{\xx}),
\end{align}
where the fundamental solution satisfying:
\begin{align*}
\int_\Omega \LL[U(\xx-\tilde{\xx})]\uu(\xx) d\xx=\uu(\tilde{\xx}),
\end{align*}
has the following analytic form:
\begin{align}
\label{eq:U_form}
U(\xx-\tilde{\xx})=&\frac{1}{8\pi\mu}
\left[\Delta r\id-\frac{\lambda+\mu}{\lambda+2\mu}\nabla^2 r\right]  \\
=&\frac{1}{16\pi\mu(1-\nu)r}\left[(3-4\nu)\id + \nabla r\nabla r^T\right]\nonumber   \\
&r\triangleq|\xx-\tilde{\xx}|\quad \nu=\frac{\lambda}{2(\lambda+\mu)}.\nonumber
\end{align}

\subsection{Body Force Term}
\prettyref{eq:BIE_FS} still involves a volume integral but we can reduce that to surface integral by using the special form of body force: $\GG(\xx)=\GG_0+\nabla\GG\xx$ and the Galerkin vector form of the fundamental solution (\prettyref{eq:U_form}). The body force term involves to two basic terms. The first one is:
\begin{align*}
 \int_\Omega \Delta r\GG d\xx
=&\int_\Omega \nabla\cdot(\nabla r\GG^T - r\nabla\GG^T) d\xx   \\
=&\int_{\partial\Omega} \left[\GG\nabla r^T - r\nabla\GG\right]\nn d\pxx.
\end{align*}
The second one is:
\begin{align*}
\int_\Omega \nabla^2 r\GG d\xx  
=&\int_\Omega \nabla\cdot(\GG\nabla r^T - \tr{\nabla\GG}r\id) d\xx   \\
=&\int_{\partial\Omega} \left[\nabla r\GG^T - \tr{\nabla\GG}r\id\right]\nn d\pxx.
\end{align*}
Plugging these two terms into \prettyref{eq:BIE_FS} and we get:
\begin{align}
\label{eq:body_force}
&\int_\Omega U^T\GG d\xx =\int_{\partial\Omega} \GGG\nn d\pxx   \\
&\GGG \triangleq\frac{1}{8\pi\mu}
\left[(\GG\nabla r^T-r\nabla\GG)-\frac{\lambda+\mu}{\lambda+2\mu}(\nabla r\GG^T-\tr{\nabla\GG}r\id)\right].\nonumber
\end{align}
\subsection{Singular Integrals}
At this step all the terms in \prettyref{eq:BIE} have been transformed into boundary integrals. However, $\xx$ in this form must be interior to $\Omega$. In this section, we take the limit of $\xx$ to $\partial\Omega$ and derive the Cauchy principle value of singular integral terms. 

The first integral in \prettyref{eq:BIE}, or the body force term in \prettyref{eq:body_force}, has removable singularity so that we can use numerical techniques to integrate them directly. The second term in \prettyref{eq:BIE} takes a special form due to our Dirac external force distribution in \prettyref{eq:elasticity}:
\begin{align*}
 &\int_{\partial\Omega} U^T\NN[\uu] d\pxx 
= -\int_{\partial\Omega} U^T\sum_{i=1}^N\delta(\xx-\xx_i)\ff_i d\pxx    \\
=&-\sum_{i=1}^NU(\xx_i-\tilde{\xx}^*)\ff_i.
\end{align*}
which is also non-singular. The third term in \prettyref{eq:BIE} is singular whose value must be determined:
\begin{align*}
 &\lmt\int_{\partial\Omega} \NN[U]^T\uu d\pxx \\
=&\lmt\int_{\partial\Omega-B(\epsilon)} \NN[U]^T\uu d\pxx+
  \lmt\int_{\partial\Omega\cap B(\epsilon)} \NN[U]^T\uu d\pxx,
\end{align*}
where we assume that $\|\tilde{\xx}-\tilde{\xx}^*\|=\epsilon$, $\tilde{\xx}^*\in\partial\Omega$, and $B(\epsilon)$ is the sphere centered at $\tilde{\xx}^*$ with radius $\epsilon$. We evaluate the two terms separately. For the first term, we have:
\begin{align*}
 &\lmt\int_{\partial\Omega-B(\epsilon)} \NN[U]^T\uu d\pxx\triangleq\DD[\uu]   \\
=&\int_{\partial\Omega} [2\mu(\MMM[U])^T\uu+(\MMM[\frac{1}{4\pi r}])\uu+\nn^T\nabla[\frac{1}{4\pi r}]\uu] d\pxx   \\
=&\int_{\partial\Omega} [2\mu U\MMM[\uu]-\frac{1}{4\pi r}\MMM[\uu]+\nn^T\nabla[\frac{1}{4\pi r}]\uu] d\pxx    \\
\MMM[]\triangleq&\nabla[]\nn^T-\nn\nabla[]^T,
\end{align*}
which is known as double layer potential and has only removable singularities. To evaluate the second term, we use the following identity:
\begin{align}
\label{eq:integral_free}
 &\lmt\int_{\partial\Omega\cap B(\epsilon)} \NN[U]^T\uu d\pxx   \\
=&\lmt\int_{\partial(\Omega\cap B(\epsilon))} \NN[U]^T\uu d\pxx-
  \lmt\int_{\partial B(\epsilon)\cap \Omega} \NN[U]^T\uu d\pxx,\nonumber
\end{align}
Again, we break this into two terms. The first term in \prettyref{eq:integral_free} is easy to evaluate using the divergence theorem:
\begin{align*}
 &\lmt\int_{\partial(\Omega\cap B(\epsilon))} \NN[U_x]^T\uu d\pxx \\
=&\lmt\int_{\partial(\Omega\cap B(\epsilon))} \nn^T\stress(U_x)^T\uu(\tilde{\xx}^*) d\pxx\\
=&-\uu(\tilde{\xx}^*)^T\lmt\int_{\Omega\cap B(\epsilon)} \LL[U_x] d\pxx=-\uu_x(\tilde{\xx}^*).
\end{align*}
The second term in \prettyref{eq:integral_free} is called the integral free term, which evaluates to:
\begin{align*}
 &\lmt\int_{\partial B(\epsilon)\cap \Omega} \NN[U]^T\uu d\pxx\triangleq-\CU\uu    \\
\CU\triangleq&\frac{\phi\id}{4\pi}-\int_{\partial(B(\epsilon)\cap\partial\Omega)} (\xx-\tilde{\xx}^*)\nn^T dl,
\end{align*}
where $\phi$ is the internal solid angle at $\tilde{\xx}^*$.
\subsection{Putting Everything Together}
Plugging all the integrals into \prettyref{eq:BIE_FS} and we have:
\begin{align*}
 \int_{\partial\Omega}\GGG\nn d\pxx 
-\sum_{i=1}^NU(\xx_i-\tilde{\xx}^*)\ff_i
-\DD[\uu]=\CU\uu(\tilde{\xx}^*),
\end{align*}
\begin{figure}[h]
\centering
\vspace{-45px}
\includegraphics[angle=-65,width=0.25\textwidth]{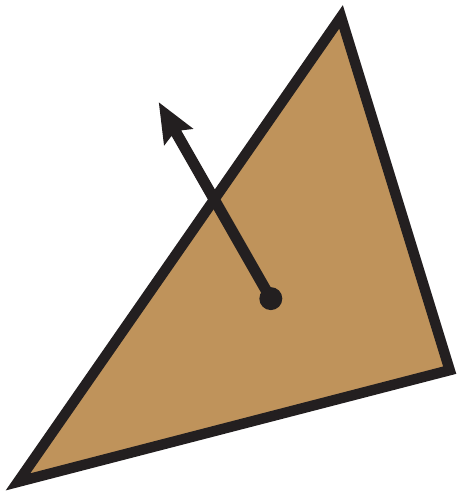}
\put(-85 ,-70){$\xx_j^1$}
\put(-7  ,-24){$\xx_j^2$}
\put(-137,-5 ){$\xx_j^3$}
\put(-40 ,0  ){$\nn_j$}
\vspace{-10px}
\caption{\label{fig:angleConstraint} The $j$th triangle.}
\vspace{-15px}
\end{figure}
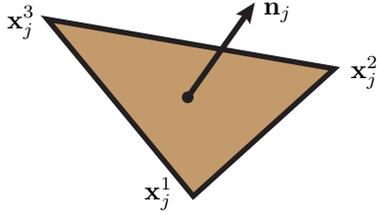
which is a dense system allowing us to solve for $\uu$ everywhere on $\partial\Omega$. This system is discretized using Galerkin's method with piecewise linear $\uu$ and piecewise constant $\ff$. All the integrals are evaluated using variable-order Gauss Quadratures. This linear system is denoted by:
\begin{align*}
(\E{D}+\CU)\uu=
\E{A}\FOURC{\GG_0}{{[\nabla\GG]}_x}{{[\nabla\GG]}_y}{{[\nabla\GG]}_z}+
\E{B}\THREEC{\ff_1}{\vdots}{\ff_N},
\end{align*}
where $\E{D}$ is the coefficient matrix of $\DD$, $\E{A}$ is the coefficient matrix of body force terms, and $\E{B}$ is the coefficient matrix of external force terms. After the displacements $\uu$ have been computed, we can recover the stress on $j$th surface triangle by solving the following linear system:
\begin{equation}
\begin{aligned}
\label{eq:stress_recover}
&\nabla\uu_j(\xx_j^2-\xx_j^1)=\uu(\xx_j^2)-\uu(\xx_j^1)    \\
&\nabla\uu_j(\xx_j^3-\xx_j^1)=\uu(\xx_j^3)-\uu(\xx_j^1)    \\
&\mu(\nabla\uu_j+\nabla\uu_j^T)\nn_j+\lambda\tr{\nabla\uu_j}\nn_j=\ff_i \\
&\stress(\xx_j)=\mu(\nabla\uu_j+\nabla\uu_j^T)+\lambda\tr{\nabla\uu_j}\id,
\end{aligned}
\end{equation}
which is 18 linear equations that can be solved for $\nabla\uu_j$ and $\stress(\xx_j)$. This linear system is denoted by:
\begin{align*}
\THREEC
{\stress_x(\xx_j)}
{\stress_y(\xx_j)}
{\stress_z(\xx_j)}=\E{N}_j\uu+\E{M}_j\THREEC{\ff_1}{\vdots}{\ff_N},
\end{align*}
where $\E{N},\E{M}$ are corressponding coefficient matrices in \prettyref{eq:stress_recover}. Combining these two systems and we can define $\BEMA_j,\BEMB_j$ as:
\begin{align*}
\BEMA_j\triangleq&\E{N}_j(\E{D}+\CU)^{-1}\E{A}  \\
\BEMB_j\triangleq&\E{N}_j(\E{D}+\CU)^{-1}\E{B}+\E{M}_j.
\end{align*}
\end{document}